\documentclass[runningheads]{llncs}

\usepackage{graphicx}
\usepackage{booktabs}
\usepackage{multicol}
\usepackage{multirow}
\usepackage{mathtools}
\usepackage{tikz}
\usepackage{comment}
\usepackage{amsmath,amssymb} 
\usepackage{color}
\usepackage{ulem} 
\usepackage[sort]{cite}
\usepackage{subcaption}
\usepackage[export]{adjustbox}

\newcommand{\I}[1]{1_{[#1]}}
\newcommand{\p}{{P}}
\newcommand{\pc}{D_c}
\newcommand{\pnc}{D_{\lnot c}}

\newcommand{\pr}{\text{R}}

\newcommand{\ours}{ECM Loss }
\newcommand{\sig}{s_c(x)}

\newcommand{\ldet}{\mathcal{L}^\mathrm{Det}_c}

\usepackage{xspace}
\newcommand{\etal}{\textit{et al.}\xspace}

\usepackage[accsupp]{axessibility}  
\usepackage{color}

\usepackage[accsupp]{axessibility}  

\begin{document}
\pagestyle{headings}
\mainmatter
\def\ECCVSubNumber{640}  

\title{Long-tail Detection with Effective Class-Margins} 


\newcommand\myparagraph[1]{\medskip \noindent{\bf #1}}

\titlerunning{Long-tail Detection with Effective Class-Margins}
%
\author{Jang Hyun Cho\inst{1} \and Philipp Krähenbühl\inst{1}}
\authorrunning{J. H. Cho and P. Krähenbühl}
%
\institute{The University of Texas at Austin, Austin TX 78712, USA \\
\email{janghyuncho7@utexas.edu, philkr@cs.utexas.edu }\\
\url{https://github.com/janghyuncho/ECM-Loss}}
\maketitle

\begin{abstract}
Large-scale object detection and instance segmentation face a severe data imbalance.
The finer-grained object classes become, the less frequent they appear in our datasets.
However, at test-time, we expect a detector that performs well for all classes and not just the most frequent ones.
In this paper, we provide a theoretical understanding of the long-trail detection problem.
We show how the commonly used mean average precision evaluation metric on an unknown test set is bound by a margin-based binary classification error on a long-tailed object detection training set.
We optimize margin-based binary classification error with a novel surrogate objective called \textbf{Effective Class-Margin Loss} (ECM).
The ECM loss is simple, theoretically well-motivated, and outperforms other heuristic counterparts on LVIS v1 benchmark over a wide range of architecture and detectors. Code is available at \url{https://github.com/janghyuncho/ECM-Loss}.
\end{abstract}


\begin{figure}[t]
\begin{center}
\begin{subfigure}{.32\textwidth}
\centering
  \includegraphics[width=\linewidth,page=1,trim=0 -1cm 0 -1.5cm]{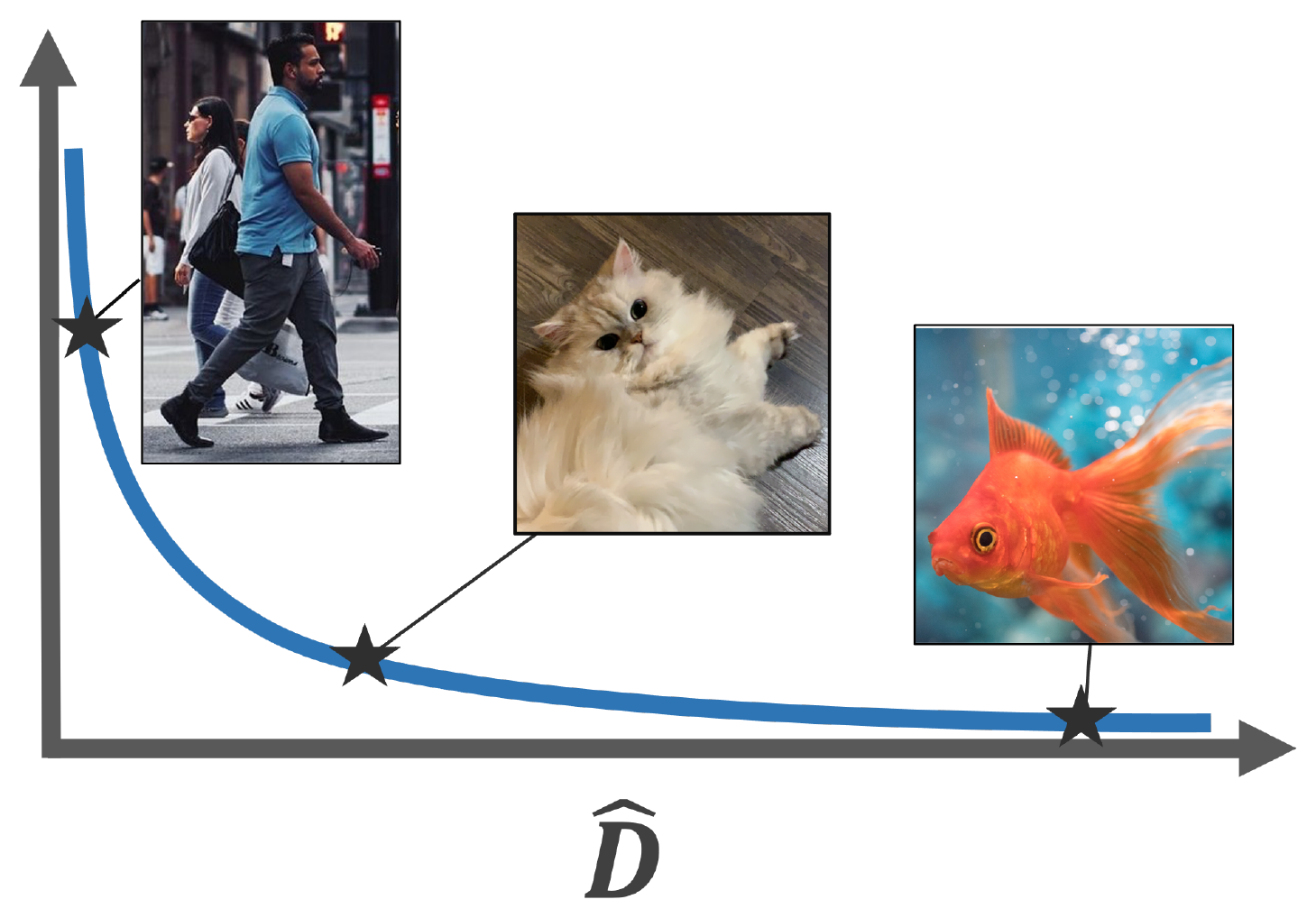}
  \caption{Data distribution}
  \label{fig:data_dist}
\end{subfigure}
\hfill
\begin{subfigure}{.32\textwidth}
\centering
  \includegraphics[width=\linewidth,page=2,trim=0 -2cm 0 -3.5cm,]{teaser_cropped.pdf}
  $\frac{1}{n}\sum_{i=1}^n \ell(s(x_i), y_i)$
  \caption{Training objective}
  \label{fig:training}
\end{subfigure}
\hfill
\begin{subfigure}{.32\textwidth}
\centering
  \includegraphics[width=\linewidth,page=3,trim=0 0 0 0cm]{teaser_cropped.pdf}
  $\frac{1}{C} \sum_{c=1}^C AP_c$
  \caption{Evaluation objective}
  \label{fig:testing}
\end{subfigure}
\end{center}
\caption{In long-trail detection, training objectives (b) do not align with evaluation objectives (c).
During training, we optimize an empirical objective on a long-tail data distribution (a).
However at test time, we expect a detector that performs well on all classes.
In this paper, we connect the detection objective (c) on an unknown test set to an empirical training objective (b) on a long-tail real-world data distribution (a) through the margin-bound theory~\cite{kakade,bartlett,panchenko,ldam}.
}
\label{fig:teaser}
\end{figure}

\section{Introduction}

The state-of-the-art performance of common object detectors has more than tripled over the past 5 years. 
However, much of this progress is measured on just 80 common object categories~\cite{coco}.
These categories cover only a small portion of our visual experiences.
They are nicely balanced and hide much of the complexities of large-scale object detection.
In a natural setting, objects follow a long-tail distribution, and artificially balancing them is hard~\cite{lvis}.
Many recent large-scale detection approaches instead balance the training loss~\cite{seesaw,eqlv2,disalign} or its gradient~\cite{eqlv1,bags} to emulate a balanced training setup.
Despite the steady progress over the past few years, these methods largely rely on heuristics or experimental discovery.
Consequently, they are often based on intuition, require extensive hyper-parameter tuning, and include a large bag-of-tricks.

In this paper, we take a statistical approach to the problem.
The core issue in long-tail recognition is that training and evaluation metrics do not line up, see Figure~\ref{fig:teaser}.
At test time, we expect the detector to do well on all classes, not just a select few.
This is reflected in the common evaluation metric: mean-average-precision (mAP)~\cite{pascal,coco,lvis,openimages,o365}. 
At training time, we ideally learn from all available data using a cross-entropy~\cite{fasterrcnn,maskrcnn,detr} 
or related loss~\cite{yolo,focal,gfocal,varifocal}. 
Here, we draw a theoretical connection between the balanced evaluation metric, mAP, and margin-based binary classification.
We show that mAP is bound from above and below by a pairwise ranking error, which in turn reduces to binary classification.
We address the class imbalance through the theory of margin-bounds~\cite{kakade,bartlett,panchenko,ldam}, 
and reduce detector training to a margin-based binary classification problem.

Putting it all together, margin-based binary classification provides a closed-form solution for the ideal margin for each object category.
This margin depends only on the number of positive and negative annotations for each object category.
At training time, we relax the margin-based binary classification problem to binary cross entropy on a surrogate objective.
We call this surrogate loss \textbf{Effective Class-Margin Loss} (ECM).
This relaxation converts margins into weights on the loss function.
Our ECM loss is completely \textit{hyperparameter-free} and applicable to a large number of detectors and backbones.

We evaluate the ECM loss on LVIS v1 and OpenImages.
It outperforms state-of-the-art large-scale detection approaches across various frameworks and backbones.
The ECM loss naturally extends to one-stage detectors~\cite{atss,centernet,varifocal,fcos}.

\section{Related works}
\myparagraph{Object detection} is largely divided into one-stage and two-stage pipelines.
In one-stage object detection~\cite{centernet,focal,yolo,yolov2,yolov3}, classification and localization are simultaneously predicted densely on each coordinate of feature map representation.
Hence, one-stage detection faces extreme foreground-background imbalance aside from cross-category imbalance.
These issues are addressed either by carefully engineered loss function such as the Focal Loss~\cite{focal,centernet}, or sampling heuristics like ATSS~\cite{atss}. 
Two-stage object detection~\cite{fastrcnn,fasterrcnn,cascadercnn} mitigates the foreground-background imbalance using a category-agnostic classifier in the first-stage. However, neither type of detection pipelines handles cross-category imbalance. 
We show that our ECM loss trains well with both types of detectors.

\myparagraph{Long-tail detection.} Learning under severely long-tailed distribution is challenging.
There are two broad categories of approaches: data-based and loss-based.
Data-based approaches include external datasets~\cite{mosaic}, extensive data augmentation with larger backbones~\cite{copypaste}, or optimized data-sampling strategies~\cite{lvis,simcal,forest}.
Loss-based approaches~\cite{eqlv1,eqlv2,seesaw,federated_loss,bags,disalign,deconfound} modify or re-weights the classification loss used to train detectors.
They perform this re-weighting either \textit{implicitly} or \textit{explicitly}.
The Equalization Loss~\cite{eqlv1} ignores the negative gradient for rare classes.
It builds on the intuition that rare classes are ``discouraged'' by all the negative gradients of other classes (and background samples).
Balanced Group Softmax (BaGS)~\cite{bags} divides classes into several {groups} according to their frequency in the training set.
BaGS then applies a cross-entropy with softmax only within each group.
This implicitly controls the negative gradient to the rare classes from frequent classes and backgrounds.
The federated loss~\cite{federated_loss} only provides negative gradients to classes that appear in an image.
This implicitly reduces the impact of the negative gradient to the rare classes.
The Equalization Loss v2~\cite{eqlv2} directly balances the ratio of cumulative positive and negative gradients per class.
The Seesaw Loss~\cite{seesaw} similarly uses the class frequency to directly reduce the weight of negative gradients for rare classes.
In addition, it compensates for the diminished gradient from misclassifications by scaling up by the ratio of the predicted probability of a class and that of the ground truth class.
These methods share the common premise that overwhelming negative gradients will influence the training dynamics of the detector and result in a biased classifier.
While this makes intuitive sense, there is little analytical or theoretical justification for particular re-weighting schemes.
This paper provides a theoretical link between commonly used mean average precision on a test set and a \textbf{weighted binary cross entropy} loss on an unbalanced training set.
We provide an optimal weighting scheme that bounds the expected test mAP.

\myparagraph{Learning with class-margins.}
Margin-based learning has been widely used in face recognition~\cite{large_margin_loss,additive_margin_loss,arcface} and classification under imbalanced data~\cite{ldam}.
In fact, assigning proper margins has a long history in bounds to generalization error~\cite{kakade,lipschitz,bartlett,panchenko}.
Cao et al.~\cite{ldam} showed the effectiveness of analytically derived margins in imbalanced classification.
In \textit{separable} two-class setting (i.e., training error can converge to 0), closed form class-margins follow from a simple constrained optimization.
Many recent heuristics in long-tail detection use this setting as the basis of re-weighted losses~\cite{seesaw,eqlv1,bags}.
We take a slightly different approach.
We show that the margin-based classification theory applies to detection by first establishing a connection between mean average precision (mAP) and a pairwise ranking error.
This ranking error is bound from above by a margin-based classification loss.
This theoretical connection then provides us with a set of weights for a surrogate loss on a \textit{training set} that optimizes the mAP on an \textit{unknown test set}.

\myparagraph{Optimizing average precision.}
Several works optimize the average precision metric directly.
Rol\'{i}nek et al.~\cite{blackboxdiff} address non-differentiability and introduced black-box differentiation.
Chen et al.~\cite{aploss} propose an AP Loss which applies an error-driven update mechanism for the non-differentiable part of the computation graph.
Oksuz et al.~\cite{alrploss} took a similar approach to optimize for Localization-Recall-Precision (LRP)~\cite{lrp}.
In contrast, we reduce average precision to ranking and then margin-based binary classification, which allows us to use tools from learning theory to bound and minimize the generalization error of our detector.


\section{Preliminary}
\label{preliminary}
\myparagraph{Test Error Bound.}
Classical learning theory bounds the test error in terms of training error and some complexity measures.
Much work builds on the Lipschitz bound of Bartlett and Mendelson~\cite{lipschitz}.
For any Lipschitz continuous loss function $\ell$, it relates the expected error $\mathcal{L}(f)=E[\ell(f(x),y)]$ and empirical error $\hat{\mathcal{L}}(f)=\frac{1}{n}\sum_{i=1}^n \ell(f(x_i),y_i)$ over a dataset of size $n$ with inputs $x_i$ and labels $y_i$.
In detection, we commonly refer to the expected error as a test error over an unknown test set or distribution, and empirical error as a training error.
Training and testing data usually follow the same underlying distribution, but different samples.


\begin{theorem}
[Class-margin bound~\cite{kakade,ldam}]
\label{lipschitz}
Suppose $\ell(x) = \I{x < 0}$ a zero-one error and $\ell_\gamma=\I{x < \gamma}$ a margin error with a non-negative margin $\gamma$. Similarly, $\mathcal{L}_y(f)=E[\ell(f(x)_y)]$ and $\mathcal{L}_{\gamma, y}(f)=E[\ell_\gamma(f(x)_y)]$. Then, for each class-conditional data distribution $P(X|Y=c)$, for all $f \in \mathcal{F}$ and class-margin $\gamma_c>0$, with probability $1-\delta_c$, the class-conditional test error for class $c\in C$ can be bounded from above as following:
\begin{align*}
\mathcal{L}_c(f)\le \hat{\mathcal{L}}_{\gamma_c, c}(f) + \frac{4}{\gamma_c} \mathcal{R}_n(\mathcal{F}) + \sqrt{\frac{\log(2/\delta_c)}{n_c}} + \epsilon(n_c, \gamma_c, \delta_c)
\end{align*}
where $\mathcal{R}_n(\mathcal{F})$ is the Rademacher complexity of a function class $\mathcal{F}$ which is typically bounded by $\sqrt{\frac{C(\mathcal{F})}{n}}$ for some complexity measure of $C$~\cite{ldam,kakade,bartlett}.
\label{thm:margin-bound}
\end{theorem}
Kakade et al.~\cite{kakade} prove the above theorem for binary classification, and Cao et al.~\cite{ldam} extend it to multi-class classification under a long-tail.
Their proof follows Lipschitz bounds of Bartlett and Mendelson~\cite{lipschitz}.
Theorem~\ref{thm:margin-bound} will be our main tool to bound the generalization error of a detector.


\myparagraph{Detection Metrics.}
Object detection measures the performance of a model through average precision along with different recall values.
Let $\mathrm{TP}(t)$, $\mathrm{FP}(t)$, and $\mathrm{FN}(t)$ be the true positive, false positive, and false negative detections for a score threshold $t$.
Let $\mathrm{Pc}(t)$ be the precision and $\mathrm{Rc}(t)$ be the recall.
Average precision $\mathrm{AP}$ then integrates precision over equally spaced recall thresholds $t \in T$:
\begin{align}
    \mathrm{Pc}(t)&=\frac{\mathrm{TP}(t)}{\mathrm{TP}(t) + \mathrm{FP}(t)}\quad \mathrm{Rc}(t) = \frac{\mathrm{TP}(t)}{\mathrm{TP}(t) + \mathrm{FN}(t)}\quad \mathrm{AP}= \frac{1}{|T|} \sum_{t\in T}\mathrm{Pc}(t).
\end{align}
Generally, true positives, false positives, and false negatives follow an assignment procedure that enumerates all annotated objects.
If a ground truth object has a close-by prediction with a score $s > t$, it counts as a positive.
Here closeness is measured by overlap.
If there is no close-by prediction, it is a false negative.
Any remaining predictions with score $s > t$ count towards false positives.
All above metrics are defined on finite sets and do not directly extend general distributions.

In this paper, we base our derivations on a probabilistic version of average precision.
For every class $c\in C$, let $\pc$ be the distribution of positive samples, and $\pnc$ be the distribution of negative samples.
These positives and negatives may use an overlap metric to ground truth annotations.
Let $P(c)$ and $P(\lnot c)$ be the prior probabilities on labels of class $c$ or not $c$.
$P(c)$ is proportional to the number of annotated examples of class $c$ at training time.
Let $s_c(x) \in [0,1]$ be the score of a detector for input $x$.
The probability of a detector $s_c$ to produce a true positive of class $c$ with threshold $t$ is $tp_c(t)=P(c) P_{x \sim \pc}(s_c(x) > t)$, false positive $fp_c(t)=P(\lnot c) P_{x \sim \pnc}(s_c(x) > t)$, and false negative $fn_c(t)=P(c)P_{x \sim \pc}(s_c(x) \le t)$.
This leads to a probabilistic recall and precision
\begin{align}
r_c(t) &= \frac{tp_c(t)}{tp_c(t)+fn_c(t)} = \frac{tp_c(t)}{P(c)} = P_{x \sim \pc}(s_c(x) > t),\\
p_c(t) &= \frac{tp_c(t)}{tp_c(t)+fp_c(t)} 
= \frac{r_c(t)}{r_c(t) + \alpha_c P_{x \sim \pnc}(s_c(x) > t)}.
\end{align}
Here, $\alpha_c = \frac{P(\lnot c)}{P(c)}$ corrects for the different frequency of foreground and background samples for class $c$.
By definition $1-r_c(t)$ is a cumulative distribution function $r_c(1)=0$, $r_c(0)=1$ and $r_c(t) \ge r_c(t+\delta)$ for $\delta > 0$.
Without loss of generality, we assume that the recall is strictly monotonous $r_c(t) > r_c(t+\delta)$\footnote{For any detector $s_c$ with a non-strict monotonous recall, there is a nearly identical detector $s^\prime_c$ with strictly monotonous recall: $s^\prime_c(x) = s_c(x)$ with chance $1-\varepsilon$ and uniform at random $s^\prime_c(x) \in U[0,1]$ with chance $\varepsilon$ for any small value $\varepsilon>0$.}.
For a strictly monotonous recall $r_c(t)$, the quantile function is the inverse $r_c^{-1}(\beta)$.
Average precision then integrates over these quantiles.

\begin{definition}[Probabilistic average precision]
\begin{equation*}
 ap_c = \int_0^1 p_c(r_c^{-1}(\beta))d\beta = \int_0^1 \frac{\beta}{\beta+\alpha P_{x \sim \pnc}(s_c(x) > r_c^{-1}(\beta))}d\beta
\end{equation*}
\end{definition}

There are two core differences between regular AP and probabilistic AP:
1) The probabilistic formulation scores a nearly exhaustive list of candidate objects, similar to one-stage detectors or the second stage of two-stage detectors.
It does not consider bounding box regression.
2) Regular AP penalizes duplicate detections as false positives, probabilistic AP does not.
This means that at training time, positives and negatives are strictly defined for probabilistic AP, which makes a proper analysis possible.
At test time, non-maxima-suppression removes most duplicate detections without major issues.

In the next section, we show how this probabilistic AP relates to a pairwise ranking error on detections.
\begin{definition}[Pairwise Ranking Error]
\label{pre_def}
\begin{align*}
\pr_c 
&= \p_{x, x^\prime \sim \pc \times \pnc} \big(s_c(x) < s_c(x^\prime) \big) = E_{x^\prime\sim \pnc}\left[1-r_c(s_c(x^\prime))\right]
\end{align*}
\end{definition}
The pairwise ranking error measures how frequently negative samples $x^\prime$ rank above positives $x$. The second equality is derived in supplement. 

While it is possible to optimize the ranking error empirically, it is hard to bound the empirical error.
We instead bound $\pr_c$ by a margin-based 0-1 classification problem and use Theorem~\ref{thm:margin-bound}.

\section{Effective Class-Margins}

We aim to train an object detector that performs well for all object classes.
This is best expressed by maximizing mean average precision over all classes equally: $mAP = \frac{1}{|C|}\sum_{c\in C}ap_c$.
Equivalently, we aim to minimize the detection error

\begin{equation}
  \mathcal{L}^\mathrm{Det} = 1 - mAP = \frac{1}{|C|}\sum_{c\in C}\underbrace{(1-ap_c)}_{\mathcal{L}^\mathrm{Det}_c}
\end{equation}

Optimizing the detection error or mAP directly is hard~\cite{aploss,alrploss,blackboxdiff,rsloss}. 
First, $ap_c$ involves a computation over the entire distribution of detections and does not easily factorize over individual samples.
Second, our goal is to optimize the expected detection error.
However, at training time, we only have access to an empirical estimate over our training set $\hat D$.

Despite these complexities, it is possible to optimize the expected detection error.
The core idea follows a series of bounds for each training class $c$:
$$
\mathcal{L}^\mathrm{Det}_c \lesssim m_c R_c \lesssim m_c \hat{\mathcal{L}}_{\gamma^\pm_c, c},
$$
where $\lesssim$ refers to inequalities up to a constant.
In Section~\ref{sec:errorbound}, we bound the detection error $\mathcal{L}^\mathrm{Det}_c$ by a weighted version of the ranking error $R_c$.
In Section~\ref{sec:rankingbound}, we directly optimize an empirical upper bound $\hat{\mathcal{L}}_{\gamma^\pm_c, c}$ to the weighted ranking error using class-margin-bounds in Theorem~\ref{thm:margin-bound}.
Finally, in Section~\ref{sec:ecm} we present a differentiable loss function to optimize the class-margin-bound.

\begin{figure}[t]
\begin{center} 
\begin{subfigure}{.32\textwidth}
  \includegraphics[width=\linewidth,page=1]{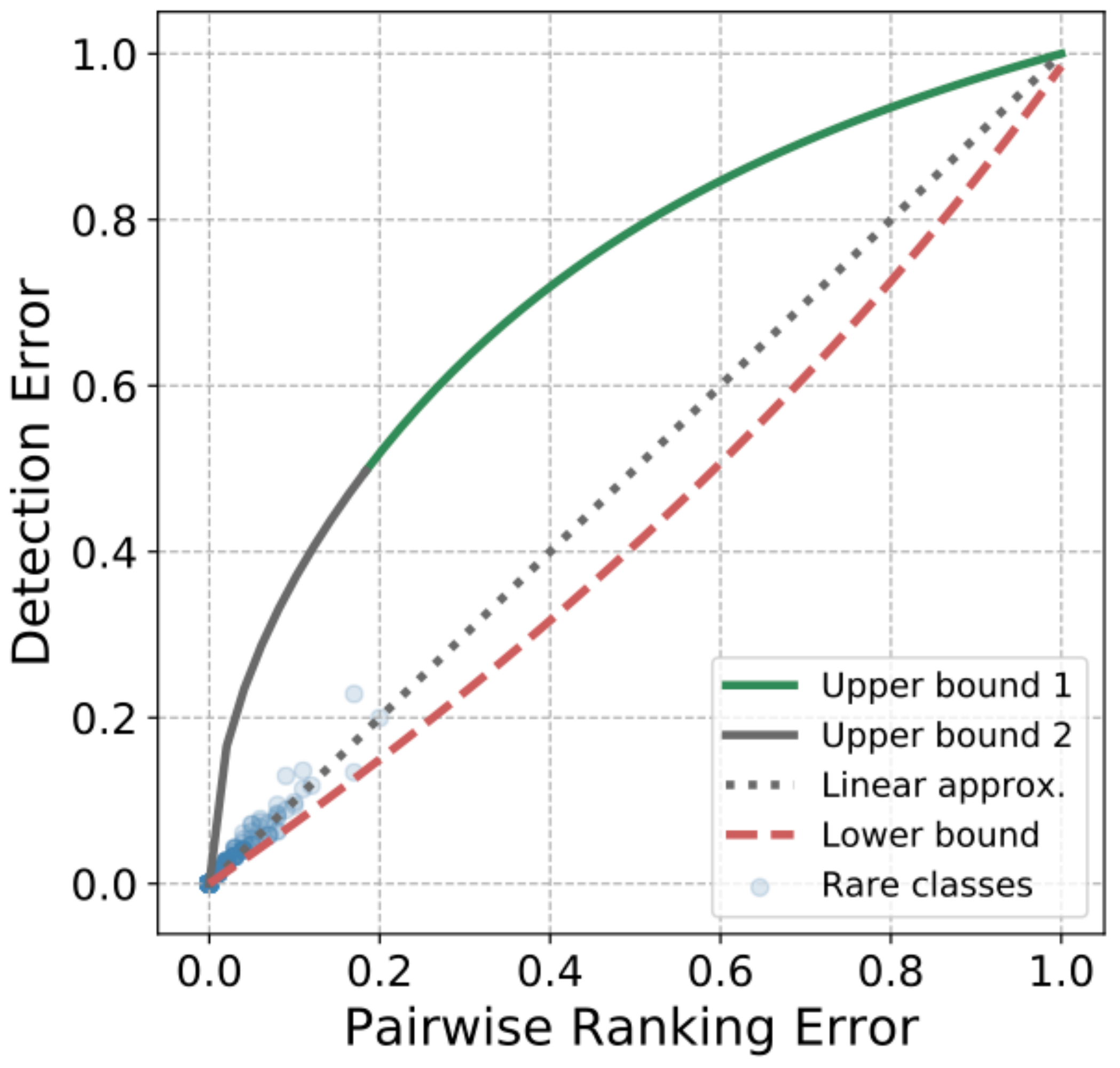} 
  \caption{Rare classes}
  \label{fig:bound_rare}
\end{subfigure}
\hfill
\begin{subfigure}{.32\textwidth}
  \includegraphics[width=\linewidth,page=2]{plot_new_small_cropped.pdf} \caption{Common classes}
  \label{fig:bound_common}
\end{subfigure}
\hfill
\begin{subfigure}{.32\textwidth}
  \includegraphics[width=\linewidth,page=3]{plot_new_small_cropped.pdf} \caption{Frequent classes}
  \label{fig:bound_frequent}
\end{subfigure}
\end{center}
\caption{Visualization of the upper and lower bound of Detection Error with respect to Pairwise Ranking Error. The solid and dotted lines are the theoretical bounds discussed in Theorem~\ref{thm:rank_bound}. We show that \textbf{actual detection errors} strictly follow the derived bounds. 
We evaluate multiple checkpoints of the same detector on rare, common, and frequent classes of lvis v1.
Each point represents a checkpoint's performance on one class.
We compute AP and ranking errors over the training set which has low errors especially for rare classes.
In practice the upper and lower bounds are tight and the linear approximation fits well.
}
\label{fig:bound_plot}
\end{figure}

\subsection{Detection Error Bound}
\label{sec:errorbound}

There is a strong correlation between the detection error $\ldet$ and the ranking objective $R_c$.
For example, a perfect detector, that scores all positives above negatives, achieves both a ranking and detection error of zero.
A detector that scores all negatives higher than positives has a ranking error of 1 and a detection error close to 1.
For other error values the the ranking error $R_c$ bounds the $ap_c$ from both above and below, as shown in Figure~\ref{fig:bound_plot} and Theorem~\ref{thm:rank_bound}.

\begin{theorem}[AP - Pairwise Ranking Bound]
\label{thm:rank_bound}
For a class $c$ with negative-to-positive ratio $\alpha_c=\frac{P(\lnot c)}{P(c)}$, the ranking error $R_c$ bounds the probabilistic average precision $ap_c$ from above and below:
\begin{align*}
 \alpha_c\log\left(\frac{1+\alpha_c}{1+\alpha_c-R_c}\right) \le \ldet \le \min \bigg(\sqrt{\frac{2}{3}\alpha_c \pr_c}, 1-\frac{8}{9} \frac{1}{1+2 \alpha_c \pr_c}\bigg).
\end{align*}
\end{theorem}
We provide a full proof in supplement and sketch out the proof strategy here.
We derive both bounds using a constrained variational problem.
For any detector $s_c$, data distributions $D_c$ and $D_{\lnot c}$, the average precision has the form
\begin{equation}
ap_c = \int_0^1 \frac{\beta}{\beta+\alpha g(\beta)} d\beta\label{ap_g},
\end{equation}
where $g(\beta) = P_{x^\prime\sim D_{\lnot c}}\left(s_c(x^\prime) > r_c^{-1}(\beta)\right) = P_{x^\prime\sim D_{\lnot c}}\left(r_c(s_c(x^\prime)) < \beta\right)$, since the recall is a strictly monotonously decreasing function.
At the same time the ranking loss reduces to  \begin{equation}
R_c = \int_0^1 g(\beta) d\beta = E_{x^\prime \sim D_{\lnot c}}\left[1-r_c(s_c(x^\prime))\right]\label{r_g}.
\end{equation}
For a fixed ranking error $R_c = \kappa$, we find a function $0\le g(\beta) \le 1$ that minimizes or maximizes the detection error $\ldet$ through variational optimization.
See supplement for more details.
See Figure~\ref{fig:bound_plot} for a visualization of the bounds.

Theorem~\ref{thm:rank_bound} clearly establishes the connection between the ranking and detection.
Unfortunately, the exact upper bound is hard to further simplify.
We instead chose a linear approximation $\ldet \approx m_c R_c$, for $\alpha_c\log\left(\frac{1+\alpha_c}{\alpha_c}\right) \le m_c \le \frac{\frac{1}{9}+2\alpha_c}{1+2\alpha_c}$.
Figure~\ref{fig:bound_plot} visualizes this linear approximation.
The linear approximation even bounds the detection error from above $\ldet \le m_c R_c + o$ with an appropriate offset $o$.
We denote this as $\ldet \lesssim m_c R_c$.

In the next section, we show how this ranking loss is bound from above with a margin-based classification problem, which we minimize in \ref{sec:ecm}.

\subsection{Ranking bounds}
\label{sec:rankingbound}

To connect the ranking loss to the generalization error, we first reduce ranking to binary classification.

\begin{theorem}[Binary error bound]
\label{binary_upper}
The ranking loss is bound from above by
\begin{align*}
\pr_c &\le \p_{x \sim \pc}\big(\sig \le t\big)+\p_{x\sim \pnc}\big(t<\sig\big),
\end{align*}
for an arbitrary threshold $t$.\label{thm:binary_bound}
\end{theorem}
\begin{proof}
For any indicator $\I{a < b} \le \I{a < t} + \I{t \le b}$.
Let's first rewrite ranking as expectations over indicator functions:
\begin{align*}
R_c &= E_{x\sim \pc}\left[E_{x^\prime\sim \pnc}\left[\I{s_c(x)<s_c(x^\prime)}\right]\right]\\
&\le E_{x\sim \pc}\left[E_{x^\prime\sim \pnc}\left[\I{s_c(x) \le t} + \I{t < s_c(x^\prime)}\right]\right]\\
&= E_{x\sim \pc}\left[\I{s_c(x) \le t}\right]+E_{x^\prime\sim \pnc}\left[\I{t < s_c(x^\prime)}\right].
\end{align*}
The last line uses the linearity of expectation.
\qed
\end{proof}

\begin{figure}[t]
\begin{center} 
\begin{subfigure}{.32\textwidth}
  \includegraphics[width=\linewidth,page=1]{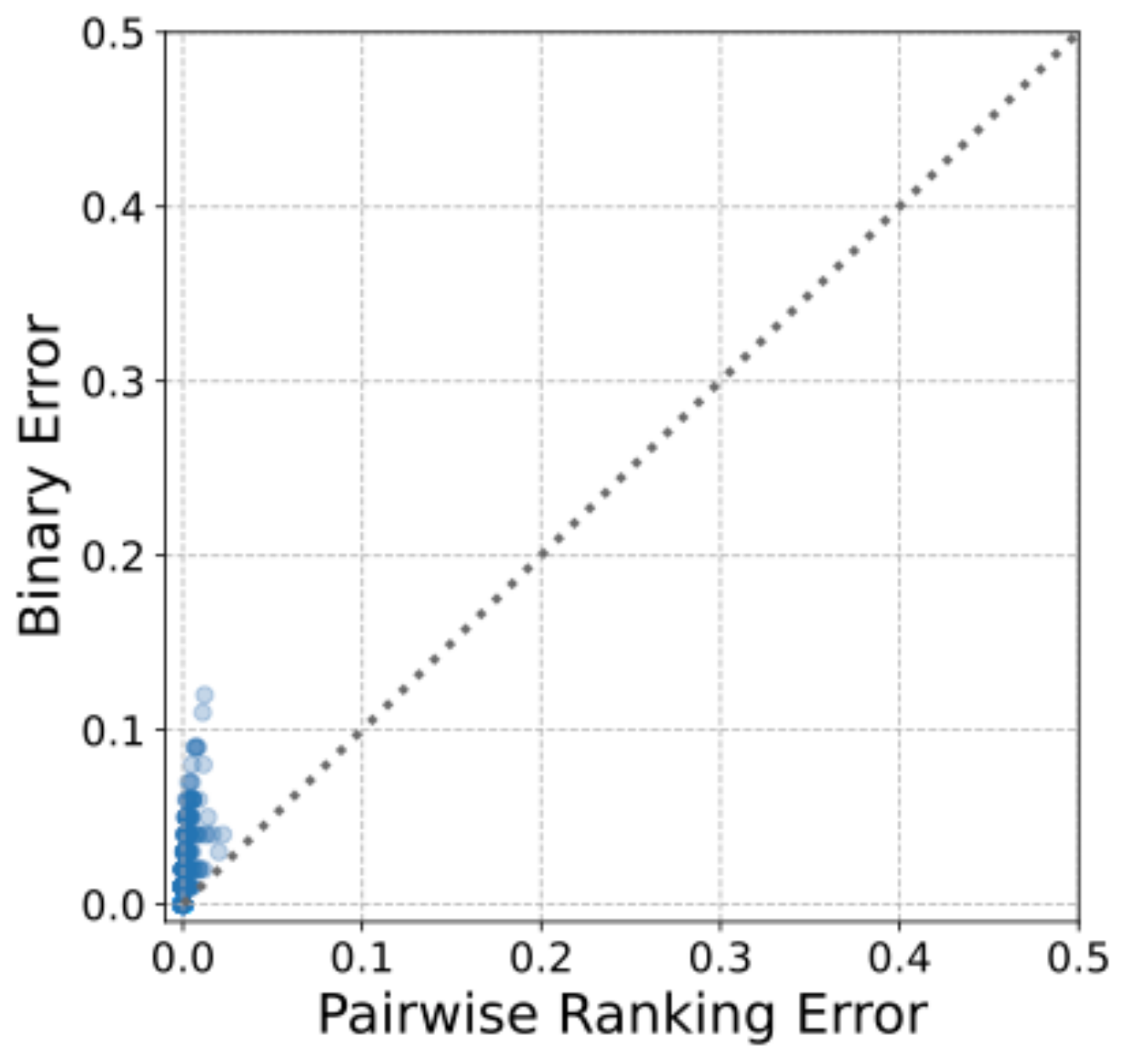} 
  \caption{Rare classes}
  \label{fig:rank_bound_rare}
\end{subfigure}
\hfill
\begin{subfigure}{.32\textwidth}
  \includegraphics[width=\linewidth,page=2]{rank_binary_small_cropped.pdf} \caption{Common classes}
  \label{fig:rank_bound_common}
\end{subfigure}
\hfill
\begin{subfigure}{.32\textwidth}
  \includegraphics[width=\linewidth,page=3]{rank_binary_small_cropped.pdf} \caption{Frequent classes}
  \label{fig:rank_bound_frequent}
\end{subfigure}
\end{center}
\caption{Visualization of the upper bound of Pairwise Ranking Error with respect to a binary classification error under an optimal threshold $t$. The blue dots correspond to actual binary errors of a detector.
We evaluate multiple checkpoints of the same detector on rare, common, and frequent classes of lvis v1.
Each point represents a checkpoint's performance in one class.
We compute classification and ranking errors over the training set which has low errors, especially for rare classes.
}
\label{fig:rank_bound_plot}
\end{figure}

Figure~\ref{fig:rank_bound_plot} visualizes this upper bound.
While any threshold $t$ leads to an upper bound to ranking.
We would like to optimize for the tightest upper bound $t$.
We do this by folding $t$ into the optimization.
In a deep network, this simply means optimizing for a bias term of the detector score $s_c(x)$.
For the remainder of the exposition, we assume $t$ is part of $s_c$ and use a detection threshold of $\frac{1}{2}$.

Next, lets us use Theorem~\ref{thm:margin-bound} to bound the classification error, and thus the detection objective, by an empirical bound

\begin{align}
\ldet \!\lesssim\! m_c \Bigg(\hat{\mathcal{L}}_{\gamma^+_c,c} \!+\! \hat{\mathcal{L}}_{\gamma^-_c,\lnot c}
\!+\! \frac{2}{\gamma^+_c}\sqrt{\frac{C(\mathcal{F})}{n_c}} \!+\! \frac{2}{\gamma^-_c}\sqrt{\frac{C(\mathcal{F})}{n_{\lnot c}}} \!+\! \epsilon(n_c) \!+\! \epsilon(n_{\lnot c})\Bigg),
\end{align}
where $\epsilon$ is a small constant that depends on the number of training samples $n_c$ and $n_{\lnot c}$.
Here, we use empirical foreground $\hat{\mathcal{L}}_{\gamma^+_c,c} = \frac{1}{n}\sum_{i=1}^{n} \I{y_i=c}\I{s_c(\hat x_i) \le \gamma_c^+}$ and background $\hat{\mathcal{L}}_{\gamma^-_c,\lnot c} = \frac{1}{n}\sum_{i=1}^{n} \I{y_i\ne c} \I{s_{\lnot c}(\hat x_i) \le \gamma_c^-}$ classification errors for detector $s_c$.
$\gamma^+_c$ and $\gamma^-_c$ are positive and negative margins respectively.

Under a separability assumption, the tightest margins take the form
\begin{align}
\gamma_c^+ = \frac{n_{\lnot c}^{1/4}}{n_c^{1/4}+n_{\lnot c}^{1/4}}\qquad
\gamma_c^- = \frac{n_{c}^{1/4}}{n_c^{1/4}+n_{\lnot c}^{1/4}}.
\label{eq:optimal_margin_eq}
\end{align}
See Cao et al.~\cite{ldam} or the supplement for a derivation of these margins.

We have now arrived at an upper bound of the detection error $\mathcal{L}^{\text{Det}} = \frac{1}{|C|}\sum_{c \in C}\ldet$ using an empirical margin-based classifier for each class $c$.
This margin-based objective takes the generalization error and any potential class imbalance into account.

In the next section, we derive a continuous loss function for this binary objective and optimize it in a deep-network-based object detection system.
Note that standard detector training is already classification-based, and our objective only introduces a margin and weight for each class.


\subsection{Effective Class-Margin Loss}
\label{sec:ecm}

Our goal is to minimize the empirical margin-based error
\begin{equation}
\hat{\mathcal{L}}_{\gamma^\pm_c\!,c} = \hat{\mathcal{L}}_{\gamma^+_c\!,c} + \hat{\mathcal{L}}_{\gamma^-_c\!,\lnot c} = \frac{1}{n}\sum_{i=1}^{n}\left( \I{y_i=c}\I{s_c(\hat x_i) \le \gamma_c^+} + \I{y_i\ne c}\I{s_{\lnot c} (\hat x_i) \le \gamma_c^-} \right)\label{eq:emp_cls}
\end{equation}
for a scoring function $s_c(x) \in [0,1]$.
A natural choice of scoring function is a sigmoid $s_c(x) = \frac{\exp(f(x))}{\exp(f(x))+\exp(-f(x))}$.

However, there is no natural equivalent to a margin-based binary cross-entropy (BCE) loss.
Regular binary cross-entropy optimizes a margin $s_c(x) = s_{\lnot c}(x) = \frac{1}{2}$ at $f(x)=0$, which does not conform to our margin-based loss.
We instead want to move this decision boundary to $s_c(x) = \gamma_c^+$ and $s_{\lnot c}(x) = \gamma_c^-$.

We achieve this with a surrogate \textbf{Effective Class-Margin Loss}:
\begin{align}
\mathcal{L}^{\text{ECM}}_c &= - \frac{1}{n}\sum_{i=1}^{n}m_c \left( 1_{[y=c]} \log (\hat s_c(x)) + 1_{[y\ne c]} \log (1-\hat s_c(x))\right).\label{eq:surrogate_loss}
\end{align}
The ECM loss optimizes a binary cross entropy on a surrogate scoring function
\begin{align}
\hat s_c(x) &= \frac{w_c^+ e^{f(x)_c}}{w_c^+ e^{f(x)_c} + w_c^- e^{-f(x)_c}}, \qquad w_c^\pm = (\gamma_c^\pm)^{-1}.\label{eq:surrogate_s}
\end{align}

This surrogate scoring function has the same properties as a sigmoid $\hat s_c \in [0, 1]$ and $\hat s_{\lnot c}(x) = 1 - \hat s_c(x)$.
However, its decision boundary $\hat s_c(x) = \hat s_{\lnot c}(x)$ lies at $f(x)=\frac{1}{2}(\log w_c^- - \log w_c^+)$.
In the original sigmoid scoring function $s$, this decision boundary corresponds to $s_c(x) = \gamma_c^+$ and $s_{\lnot c}(x) = \gamma_c^-$.
Hence, the \textbf{Effective Class-Margin Loss} minimizes the binary classification error under the margins specified by our empirical objective~\eqref{eq:emp_cls}. In Figure~\ref{fig:mw_plots}, we visualize the relationship between the negative-to-positive ratio $\alpha_c$ and the positive and negative class margins/weights. 

We use this ECM loss as a plug-in replacement to the standard binary cross entropy or softmax cross entropy used in object detection.

\begin{figure}[t]
\begin{center}
\begin{subfigure}{.24\textwidth}
  \includegraphics[width=\linewidth,page=1]{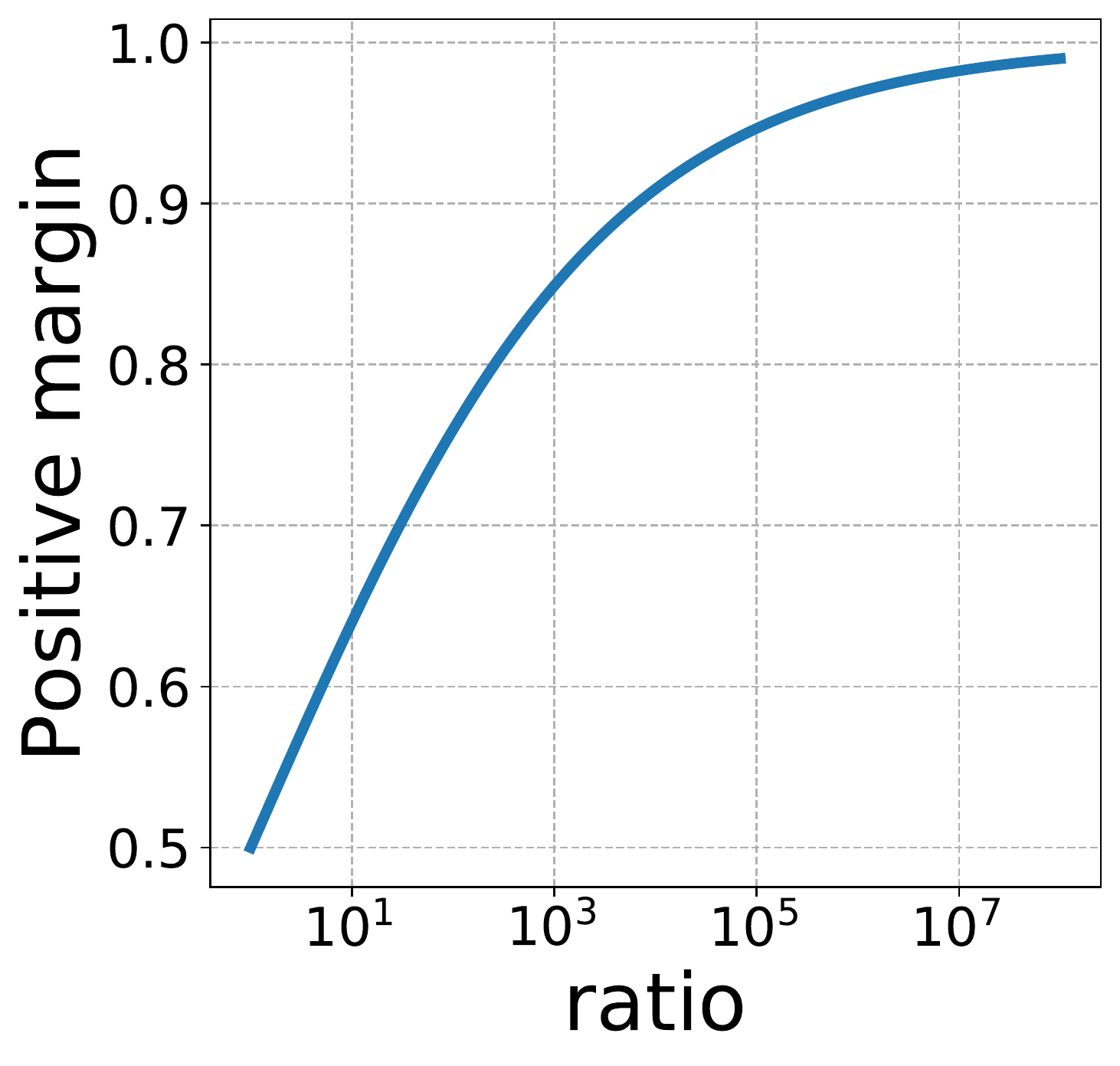}
\end{subfigure}
\hfill
\begin{subfigure}{.24\textwidth}
  \includegraphics[width=\linewidth,page=2]{pos_m.pdf}
\end{subfigure}
\hfill
\begin{subfigure}{.24\textwidth}
  \includegraphics[width=\linewidth,page=3]{pos_m.pdf}
\end{subfigure}
\hfill
\begin{subfigure}{.24\textwidth}
  \includegraphics[width=\linewidth,page=4]{pos_m.pdf}
\end{subfigure}
\end{center}
\caption{Visualization of the positive and negative margins and weights as a function of the negative-to-positive ratio $\alpha_c$. }
\label{fig:mw_plots}
\end{figure}



\begin{table}[t]
\begin{center}
\caption{LVIS v1 validation set results. We compare different methods on various frameworks and backbones on $2\times$ schedule. For Swin-B backbones~\cite{swin}, we use ImageNet-21k pretrained weight as initialization. We used the results of the original papers if available, and reproduced them from the official code otherwise. }
\label{tab:table_main}
\resizebox{\linewidth}{!}{
\begin{tabular}{cccccccc}
Framework & Backbone &  Method  & $\text{mAP}_{\text{segm}}$ & $\text{AP}_{\text{r}}$ & $\text{AP}_{\text{c}}$& $\text{AP}_{\text{f}}$ & $\text{mAP}_{\text{bbox}}$  \\
\toprule
\multirow{5}{*}{Mask R-CNN}&
\multirow{5}{*}{ResNet-50} & CE Loss & 22.7 & 10.6 & 21.8 & 29.1 & 23.3\\
& & Federated Loss~\cite{federated_loss} & 26.0 & 18.7 & 24.8 & 30.6 & 26.7  \\
& & Seesaw Loss~\cite{seesaw} & 26.7 & 18.0 & 26.5 & \textbf{32.4} & 27.3\\
& & LOCE~\cite{loce} & 26.6 & 18.5& 26.2 &30.7 & 27.4 \\
& & \ours  & \textbf{27.4} &\textbf{ 19.7} & \textbf{27.0} & 31.1 &\textbf{ 27.9} \\
\midrule
\multirow{8}{*}{Mask R-CNN}&
\multirow{8}{*}{ResNet-101} & CE Loss & 25.5 & 16.6 & 24.5 & 30.6 & 26.6\\
& & EQL v1~\cite{eqlv1}& 26.2 & 17.0 & 26.2 & 30.2  & 27.6 \\ 
& & BAGS~\cite{bags} &  25.8 & 16.5 & 25.7 & 30.1 & 26.5\\
& & EQL v2~\cite{eqlv2} & 27.2 & 20.6 & 25.9 & 31.4 & 27.9 \\
& & Federated Loss~\cite{federated_loss} & 27.9 & 20.9 & 26.8 & 32.3 & 28.8  \\
& & Seesaw Loss~\cite{seesaw} &28.1 & 20.0 & \textbf{28.0} & 31.8 & 28.9\\
& & LOCE~\cite{loce} & 28.0 & 19.5 & 27.8 & 32.0 & 29.0 \\
& & \ours & \textbf{28.7} & \textbf{21.9 }& 27.9 &\textbf{ 32.3 }& \textbf{29.4}\\
\midrule
\multirow{8}{*}{Cascade Mask R-CNN}&
\multirow{8}{*}{ResNet-101} & CE Loss  & 27.0 & 16.6 & 26.7 & 32.0 & 30.3\\
& & EQL v1~\cite{eqlv1} &  27.1 & 17.0 & 27.2 & 31.4 & 30.4 \\ 
& & De-confound TDE~\cite{deconfound} & 27.1 & 16.0 & 26.9 & 32.1 & 30.0 \\
& & BAGS~\cite{bags} & 27.0 & 16.9 & 26.9 & 31.7  & 30.2\\
& & Federated Loss~\cite{bags} & 28.6 & 20.3 & 27.5 & 33.4 & 31.8 \\
& & DisAlign~\cite{disalign} & 28.9 & 18.0 & 29.3 & 33.3 & 32.7\\
& & Seesaw Loss~\cite{seesaw} & 30.1 &\textbf{ 21.4} & 30.0 & 33.9 & 32.8\\
& & \ours & \textbf{30.6 }& 19.7 & \textbf{30.7} &\textbf{ 35.0} & \textbf{33.4} \\
\midrule
\multirow{2}{*}{Cascade Mask R-CNN}&
\multirow{2}{*}{Swin-B} & Seesaw Loss~\cite{seesaw}  &  38.7 &\textbf{34.3} & 39.6 & 39.6 & 42.8 \\
& & \ours & \textbf{39.7 }& 33.5 & \textbf{40.6} &\textbf{ 41.4 }& \textbf{43.6} \\
\bottomrule
\end{tabular}
}
\end{center}

\end{table}

\section{Experiments}
\subsection{Experimental Settings}
\myparagraph{Datasets.}
We evaluate our method on LVIS v1.0~\cite{lvis} and OpenImages datasets.
LVIS v1.0 is large-scale object detection and instance segmentation dataset.
It includes 1203 object categories that follow the extreme long-tail distribution.
Object categories in LVIS dataset are divided into three groups by frequency: \textit{frequent}, \textit{common}, and \textit{rare}.
Categories that appear in less than 10 images are considered rare, more than 10 but less than 100 are common, and others are frequent.
There are about 1.3 M instances in the dataset over 120k images (100k train, 20k validation split).
The OpenImages Object Detection dataset contains 500 object categories over 1.7 M images in a similar long-tail.

\myparagraph{Evaluation. } We evaluate all models using both the conventional mAP evaluation metric and the newly proposed $\text{mAP}_{\text{fixed}}$ metric~\cite{ap_fixed}.
$\text{mAP}$ measures the mean average precision over IoU thresholds from 0.5 to 0.95~\cite{coco} over 300 detections per image.
$\text{mAP}_{\text{fixed}}$~\cite{ap_fixed} has no limit for detections per image, but instead limits the number of detections per class over the entire dataset to 10k.
Due to memory limitations, we also limit detections per image to 800 per class.
We further evaluate the boundary IoU $\text{mAP}_{\text{boundary}}$~\cite{boundary_iou}, the evaluation metric in this year's LVIS Challenge.
For OpenImages, we follow the evaluation protocol of Zhou~\etal~\cite{unidet} and measure mAP@0.5.


\subsection{Implementation Details.} 
Our implementation is based on Detectron2~\cite{detectron2} and MMDetection~\cite{mmdet}, two most popular open-source libraries for object detection tasks.
We train both Mask R-CNN~\cite{maskrcnn} and Cascade Mask R-CNN~\cite{cascadercnn} with various backbones: ResNet-50 and ResNet-101~\cite{resnet} with Feature Pyramid Network~\cite{fpn}, and the Swin Transformer~\cite{swin}.
We use a number of popular one-stage detectors: FCOS~\cite{fcos}, ATSS~\cite{atss} and VarifocalNet~\cite{varifocal}.
We largely follow the standard COCO and LVIS setup and hyperparameters for all models.
For OpenImages, we follow the setup of Zhou~\etal~\cite{unidet}.
More details are in the supplementary material.

\myparagraph{ECM Loss.} 
\label{paragraph:ecm_detail}
Our ECM Loss is a plug-in replacement to the sigmoid function used in most detectors.
Notably, ECM Loss does not require any hyper-parameter.
We use the training set from each dataset to measure $\alpha_c, n_c, n_{\lnot c}$ of each class.

\begin{table}[t]
\begin{center}
\caption{Comparison on LVIS v1 validation set. Models are trained with Mask R-CNN with ResNet-50 backbone on 1x schedule.
Numbers with $^*$ use a different implementation~\cite{ap_fixed}.
$\text{mAP}^{\text{fixed}}_{\text{boundary}}$ and $\text{mAP}^{\text{fixed}}_{\text{bbox}}$ refer to the new LVIS Challenge evaluation metrics~\cite{boundary_iou,ap_fixed}.}
\label{tab:table_1x}
\resizebox{\linewidth}{!}{
\begin{tabular}{lccccccc}
Method  & $\text{mAP}_{\text{segm}}$ & $\text{AP}_{\text{r}}$ & $\text{AP}_{\text{c}}$& $\text{AP}_{\text{f}}$ & $\text{mAP}_{\text{bbox}}$ &$\text{mAP}^{\text{fixed}}_{\text{boundary}}$ &$\text{mAP}^{\text{fixed}}_{\text{bbox}}$  \\
\toprule
RFS+CE Loss & 21.7 & 9.5 & 21.1 & 27.7 & 22.2 & 18.3& 25.7 \\
LWS~\cite{decouple} & 17.0 & 2.0 & 13.5 & 27.4 & 17.5 & -&-\\
cRT~\cite{decouple} & 22.1 & 11.9 & 20.2 & 29.0 & 22.2 &-&-\\
BAGS~\cite{bags} & 23.1 & 13.1 & 22.5 & 28.2 & 23.7&-&26.2$^*$\\
EQL v2~\cite{eqlv2} & 23.9 & 12.5 & 22.7 & 30.4 & 24.0 & 20.3 & 25.9 \\ 
Federated Loss~\cite{federated_loss}  &23.9 & 15.8 & 23.3 & 30.7 & 24.9&-&26.3$^*$\\
Seesaw Loss~\cite{seesaw}  & 25.2 & 16.4 & 24.4 &\textbf{ 30.8} & 25.4 & 19.8 & 26.5 \\
\midrule
\ours & \textbf{26.3} & \textbf{19.5}& \textbf{26.0} & 29.8 & \textbf{26.7} & \textbf{21.4} & \textbf{27.4 }\\
\bottomrule
\end{tabular}
}
\end{center}
\end{table}

\subsection{Experimental Results}
Table~\ref{tab:table_main} compares our approach on frameworks and backbones using a standard $2\times$ training schedule.
We compare different long-tail loss functions under different experimental setups.
With Mask R-CNN on a ResNet-50 backbone, our ECM Loss outperforms all alternative losses by $0.7$ $\text{mAP}_{\text{segm}}$ and $0.5$ $\text{mAP}_{\text{bbox}}$.
With Mask R-CNN on a ResNet-101 backbone, our ECM loss outperforms alternatives with a $0.6$ $\text{mAP}_{\text{segm}}$ and $0.4$ $\text{mAP}_{\text{bbox}}$.
The results also hold up in the more advanced Cascade R-CNN framework~\cite{cascadercnn} with ResNet-101 and Swin-B backends.
Here the gains are $0.5$ $\text{mAP}_{\text{segm}}$ and $0.6$ $\text{mAP}_{\text{bbox}}$ for ResNet-101, and $1$ $\text{mAP}_{\text{segm}}$ and $0.8$ $\text{mAP}_{\text{bbox}}$ for Swin-B.
The overall gains over a simple cross-entropy baseline are 3-5 mAP throughout all settings.
The consistent improvement in accuracy throughout all settings highlights the empirical efficacy of our method, in addition to the grounding in learning theory.


For reference, Table~\ref{tab:table_1x} compares our method on Mask R-CNN with ResNet-50 backbone on $1\times$ schedule.
Our ECM Loss outperforms all prior approaches by 1.1 $\text{mAP}_{\text{segm}}$ and 1.3  $\text{mAP}_{\text{bbox}}$.
Our method achieves a 10 mAP gain over cross-entropy loss baseline, and 3.1 $\text{AP}_{\text{r}}$ gain over the state-of-the-art.
Our method shows similar gains on the new evaluation metrics $\text{mAP}^{\text{fixed}}_{\text{boundary}}$ and $\text{mAP}^{\text{fixed}}_{\text{bbox}}$ whereas prior methods tend show a more moderate improvement on new metrics.
This improvement is particularly noteworthy, as our approach uses no additional hyperparameters, and is competitive out of the box.

\begin{table}[t]
\begin{center}
\caption{One-stage object detection results on LVIS v1 validation set. We compare popular one-stage detectors with ResNet-50 and ResNet-101 backbones, on 1x schedule. }
\label{tab:table_1stage_1x}
\resizebox{0.8\linewidth}{!}{
\begin{tabular}{ccccccc}
Framework & Backbone & Method  & $\text{AP}_{\text{r}}$ & $\text{AP}_{\text{c}}$& $\text{AP}_{\text{f}}$ & $\text{mAP}_{\text{bbox}}$  \\
\toprule
\multirow{2}{*}{FCOS}&
\multirow{2}{*}{ResNet-50} & Focal Loss~\cite{focal} & 11.2 & 21.0& \textbf{27.8} & 22.0\\
& & \ours  &  \textbf{14.5} & \textbf{22.7 }& 27.6 & \textbf{23.2}\\
\midrule
\multirow{2}{*}{FCOS}&
\multirow{2}{*}{ResNet-101} & Focal Loss~\cite{focal} & 14.1 & 22.6 & \textbf{29.8} & 24.0 \\
& & \ours  & \textbf{17.2} &\textbf{ 24.2} & 29.6 & \textbf{25.1} \\
\midrule
\multirow{2}{*}{ATSS}&
\multirow{2}{*}{ResNet-50} & Focal Loss~\cite{focal} & 8.6 & 20.5 & \textbf{29.8} & 22.1\\
& & \ours  & \textbf{15.8} & \textbf{23.5} & 29.5 & \textbf{24.5}  \\
\midrule
\multirow{2}{*}{ATSS}&
\multirow{2}{*}{ResNet-101} & Focal Loss~\cite{focal} & 12.9 & 24.0 & \textbf{31.9 }& 25.2 \\
& & \ours  & \textbf{17.7 }&\textbf{ 25.6 }& 31.5 & \textbf{26.5} \\
\midrule
\multirow{2}{*}{VarifocalNet}& 
\multirow{2}{*}{ResNet-50} & Varifocal Loss~\cite{varifocal} & 14.2 & 23.6 & \textbf{30.7} & 24.8\\
& & \ours  & \textbf{17.1 } & \textbf{25.5} & 29.7 & \textbf{25.7 }\\
\bottomrule
\end{tabular}
}
\end{center}
\end{table}

Table~\ref{tab:table_1stage_1x} compares our ECM Loss with baseline losses on FCOS, ATSS and VarifocalNet, trained with ResNet-50 and ResNet-101 backbones on 1x schedule.
The ECM Loss shows consistent gains over Focal Loss and its variants.
With FCOS, ECM improves box mAP 1.2 and 1.1 points, respectively, for ResNet-50 and ResNet-101 backbones.
With ATSS, ECM Loss improves Focal Loss {7.2} and {4.8} points on $\text{AP}_{\text{r}}$, and {2.4} and {1.3} points mAP, respectively.
We further test on VarifocalNet, a recently proposed one-stage detector, and show similar advantage using the ResNet-50 backbone.
Our ECM loss consistently improves the overall performance of a one-stage detector, especially in rare classes.
It thus serves as a true plug-in replacement to the standard cross-entropy or focal losses.
\begin{table}[t]
\begin{center}
\caption{One-stage object detection results on LVIS v1 validation set. We compare different methods with ResNet-50 backbone on 2x schedule.}
\label{tab:table_1stage_2x}
\resizebox{0.7\linewidth}{!}{
\begin{tabular}{ccccccc}
Framework & Backbone & Method  & $\text{AP}_{\text{r}}$ & $\text{AP}_{\text{c}}$& $\text{AP}_{\text{f}}$ & $\text{mAP}_{\text{bbox}}$  \\
\toprule
\multirow{2}{*}{FCOS}&
\multirow{2}{*}{ResNet-50} & Focal Loss & 12.0 & 22.9 & \textbf{29.5} & 23.5 \\
& & \ours  &\textbf{14.7} & \textbf{23.0} & \textbf{29.5} &\textbf{24.4}\\
\midrule
\multirow{2}{*}{ATSS}&
\multirow{2}{*}{ResNet-50} & Focal Loss & 14.5&24.3 &\textbf{31.8}&25.6 \\
& & \ours  &  \textbf{16.6} &\textbf{25.2}&31.3 & \textbf{26.1}\\
\bottomrule
\end{tabular}
}
\end{center}
\end{table}
Table~\ref{tab:table_1stage_2x} further analyze our ECM Loss with Focal Loss on FCOS and ATSS trained with ResNet-50 backbone on 2x schedule. Our ECM maintains improvement of 0.9 point mAP for FCOS, and 0.5 point mAP for ATSS. For $\text{AP}_{\text{r}}$, both methods consistently improve 2.7 and 2.1 points, respectively.

\begin{table}[t]
\begin{center}
\caption{Comparisons of ECM Loss on OpenImages dataset following the evaluation protocol of Zhou~\etal~\cite{unidet}. We compare using a Cascade R-CNN with ResNet-50 backbone.}
\label{tab:table_oid}
\resizebox{0.7\linewidth}{!}{
\begin{tabular}{ccccccc}
Framework & Backbone & Method & Schedule & $\text{mAP}$  \\
\toprule
\multirow{2}{*}{Cascade R-CNN}&
\multirow{2}{*}{ResNet-50} & EQL + Hier.~\cite{eqlv1,unidet} & 2x & 64.6 \\
& & ECM Loss & 2x & \textbf{65.8} \\
\bottomrule
\end{tabular}
}
\end{center}
\end{table}
In Table~\ref{tab:table_oid}, we compare ECM Loss with a variant of Equalization Loss on the class hierarchy of Zhou~\etal~\cite{unidet}.
Although OpenImages have a long-tail distribution of classes, the number of classes and the associated prior probabilities are very different.
Nevertheless, ECM Loss improves over the baseline for 1.2 mAP. 
This result confirms the generality of our method. 



For all our experiments, the class frequencies were measured directly from the annotation set of LVIS v1 training dataset.
Note that for each class, the negative sample not only includes other foreground classes but also the background class.
However, the prior probability for background class is not defined apriori from the dataset itself since it solely depends on the particular detection framework of choice.
Hence, we measure the background frequency for each detector of choice and factor it into the final derivation of overall class frequencies.
This can be done within the first few iterations during training.
We then compute the effective class-margins with the derived optimal solution in Eqn.~\eqref{eq:optimal_margin_eq} and finally define the surrogate scoring function~\eqref{eq:surrogate_s}. 

\section{Conclusion}
In this paper, we tackle the long-tail object detection problem using a statistical approach.
We connect the training objective and the detection evaluation objective in the form of margin theory.
We show how a probabilistic version of average precision is optimized using a ranking and then margin-based binary classification problem.
We present a novel loss function, called Effective Class-Margin (ECM) Loss, to optimize the margin-based classification problem.
This ECM loss serves as a plug-in replacement to standard cross-entropy-based losses across various detection frameworks, backbones, and detector designs.
The ECM loss consistently improves the performance of the detector in a long-tail setting.
The loss is simple and hyperparameter-free.

\myparagraph{Acknowledgments.} This material is in part based upon work supported by the National Science Foundation under Grant No. IIS-1845485 and IIS-2006820.

\clearpage

\bibliographystyle{splncs04}
\bibliography{main}

\newpage

\appendix
\section{Pairwise Ranking Error}
In this section, we will prove the second equality of Definition~\ref{pre_def}. 
\begin{align*}
    R_c
&=P_{x^\prime\sim D_{\lnot c},x\sim D_{c}}\left(s_x(x) < s_c(x^\prime))\right)\\
&= 1-E_{x^\prime\sim D_{\lnot c}}\left[P_{x\sim D_{c}}\left(s_x(x) > s_c(x^\prime))\right)\right]\\
&=E_{x^\prime\sim D_{\lnot c}}\left[1-r_c(s_c(x^\prime))\right]\\
&=\int^1_0 \underbrace{P_{x^\prime\sim D_{\lnot c}}(r_c(s_c(x^\prime)) < \beta)}_{=g(\beta)}\mathrm{d}\beta = \tau.
\end{align*}
where the definition of $g$ is
$$
g(\beta) = P_{x^\prime\sim D_{\lnot c}}\left(s_c(x^\prime) > r_c^{-1}(\beta)\right) = P_{x^\prime\sim D_{\lnot c}}\left(r_c(s_c(x^\prime)) \boldsymbol{<} \beta\right).
$$
The above derivation connects the ranking error to $g$ and the recall.

\section{AP - Pairwise Ranking Error Bound}
In this section, we will prove Theorem~\ref{thm:rank_bound}.
We first derive the and lower bounds to the variational objective $\int^1_0 \frac{x}{x+\alpha g(x)} \mathrm{d}x$ under constraint $\int^1_0 g(x)\mathrm{d}x = \tau$ for a function $g(x) \ge 0$.
The AP bounds then directly reduce to the variational objective.

\begin{lemma}
\label{eq:varational_lower_bound}
Consider the following variational problem
\begin{align*}
\mathrm{minimize}_g &\int^1_0 \frac{x}{x+\alpha g(x)} \mathrm{d}x \\
\mathrm{subject\:to} &\int^1_0 g(x)\mathrm{d}x = \tau \\
&g(x) \ge 0 
\end{align*}
The solution to this problem is 
\begin{align*}
\max\Bigg(1-\sqrt{\frac{2}{3}\alpha \tau}, \frac{4}{9}\frac{1}{\frac{1}{2}+\alpha \tau}\Bigg)
\end{align*}
\begin{proof}
Consider the associated Euler-Lagrangian equation:
\begin{align*}
L(x, v(x), \lambda) = &\int^1_0 \frac{x}{x+\alpha v(x)^2}\mathrm{d}x + \lambda \Bigg(\int^1_0 v(x)^2\mathrm{d}x - \tau \Bigg)
\end{align*}
where $g(x)=v(x)^2$ for the non-negativity constraint. To solve for minima
\begin{align*}
\frac{d}{d v(x)}L(x, v(x), \lambda) &= - \frac{\alpha x v(x)}{(x+\alpha v(x)^2)^2} + \lambda v(x) = 0\\
v(x)\alpha x &= \lambda v(x)(x+\alpha v(x)^2)^2\\
\Longrightarrow v(x) &= 0 \quad \text{or}\quad x + \alpha v(x)^2 = \sqrt{\frac{x\alpha}{\lambda}}=\sqrt{x}\sqrt{\frac{\alpha}{\lambda}}\\
\Longrightarrow v(x)^2 &= \max \Bigg(0, \sqrt{\frac{x}{\alpha \lambda}}-\frac{x}{\alpha}\Bigg)= 1_{[x \le \frac{\alpha}{\lambda}]}\Bigg( \sqrt{\frac{x}{\alpha \lambda}}-\frac{x}{\alpha}\Bigg)\\
\Longrightarrow \int^1_0\alpha v(x)^2 \mathrm{d}x &= \alpha \tau =\int^1_0 1_{[x \le \frac{\alpha}{\lambda}]}\Bigg( \sqrt{\frac{\alpha x}{ \lambda}}-x\Bigg)\mathrm{d}x\\
&= \int^\kappa_0\Bigg( \sqrt{\frac{\alpha x}{ \lambda}}-x\Bigg)\mathrm{d}x = \frac{2}{3}\kappa ^{\frac{3}{2}}\sqrt{\frac{\alpha}{\lambda}}-\frac{\kappa^2}{2}
\end{align*}
where $\kappa = \min(1, \frac{\alpha}{\lambda})$. For $\kappa = 1$:
\begin{align*}
 \alpha \tau &= \frac{2}{3}\sqrt{\frac{\alpha}{\lambda}} - \frac{1}{2} \Longrightarrow \sqrt{\frac{\alpha}{\lambda}} = \frac{3}{2}\bigg(\alpha \tau +\frac{1}{2}\bigg)\\
 \Longrightarrow x + \alpha v(x)^2 &= \sqrt{x}\sqrt{\frac{\alpha}{\lambda}}=\sqrt{x}\frac{3}{2}\bigg(\alpha\tau +\frac{1}{2}\bigg)\\
 \Longrightarrow \int^1_0 \frac{x}{x+\alpha v(x)^2}\mathrm{d}x &= \int^1_0 \frac{x}{\sqrt{x}\frac{3}{2}(\alpha \tau + \frac{1}{2})}\mathrm{d}x=\int^1_0 \sqrt{x}\mathrm{d}x \frac{1}{\frac{3}{2}(\frac{1}{2} + \alpha \tau)} \\
 &= \frac{4}{9}\frac{1}{\frac{1}{2}+\alpha \tau}
\end{align*}
For $\kappa < 1$:
\begin{align*}
\alpha \tau
= \frac{2}{3} \bigg(\frac{\alpha}{\lambda}\bigg)^{\frac{3}{2}}\sqrt{\frac{\alpha}{\lambda}} - \frac{1}{2}\bigg(\frac{\alpha}{\lambda}\bigg)^2 &= \frac{1}{6}\bigg(\frac{\alpha}{\lambda}\bigg)^2\Longrightarrow \lambda = \frac{1}{6}\sqrt{\frac{\alpha}{\tau}}\\
\Longrightarrow \int^1_0 \frac{x}{x+\alpha v(x)^2}\mathrm{d}x &= \int^{\kappa}_0 \frac{x}{\sqrt{\frac{\alpha x}{\lambda}}}\mathrm{d}x + \int^1_{\kappa} \frac{x}{x + 0} \mathrm{d}x\\
&= \sqrt{\frac{\lambda}{\alpha}} \int^{\kappa}_0 \sqrt{x}\mathrm{d}x + \int^1_{\kappa} \mathrm{d}x\\
&= \frac{2}{3}\sqrt{\frac{\lambda}{\alpha}} \kappa^{\frac{3}{2}} + 1 - \kappa\\
&=\frac{2}{3}\sqrt{\frac{\lambda}{\alpha}} \bigg(\frac{\alpha}{\lambda}\bigg)^{\frac{3}{2}} + 1 - \frac{\alpha}{\lambda}\\
&=1-\frac{1}{3}\frac{\alpha}{\lambda}\\
&=1-\sqrt{\frac{2}{3}\alpha\tau}
\end{align*}
Each case yields on lower bound, hence the combined lower bound is
$$
\max\Bigg(1-\sqrt{\frac{2}{3}\alpha \tau}, \frac{4}{9}\frac{1}{\frac{1}{2}+\alpha \tau}\Bigg)
$$
\qed
\end{proof}
\end{lemma}

\noindent Bonus: The two bounds meet at $\frac{2}{3}$:
\begin{align*}
\frac{4}{9}\frac{1}{\frac{1}{2}+\alpha\tau} = 1-\sqrt{\frac{2}{3}\alpha\tau} = \frac{2}{3} \quad\text{for}\quad \alpha \tau = \frac{1}{6}.
\end{align*}

\begin{lemma}
\label{eq:varational_upper_bound}
Consider the following variational problem
\begin{align*}
\mathrm{maximize}_g &\int^1_0 \frac{x}{x+\alpha g(x)} \mathrm{d}x \\
\mathrm{subject\:to} &\int^1_0 g(x)\mathrm{d}x = \tau \\
&g(x) \ge 0 
\end{align*}
\begin{proof}
First, let us re-formulate the problem as following
\begin{align*}
\mathrm{minimize}_g &\int^1_0 \frac{\alpha g(x)}{x+\alpha g(x)} \mathrm{d}x \\
\mathrm{subject\:to} &\int^1_0 g(x)\mathrm{d}x = \tau \\
&g(x) \ge 0 
\end{align*}
Without the equality constraint $\int^1_0g(x)\mathrm{d}x=\tau$, the objective is minimized at $g(x)=0$ for all $x\in [0, 1]$.
The equality constraint assigns certain values $g(x)$ a positive mass.
The optimal solution will assign $g(x)=0$ for $x < 1-\tau$, and $g(x)=1$ for $x \ge 1-\tau$.
To see this, consider a value $g(x_1)=\frac{\epsilon}{\alpha}$ for $x_1 < 1-\tau$ and one or move values $g(x_2) \le 1-\frac{\epsilon}{\alpha}$ for $x_2 > 1-\tau$.
Here, a solution $\hat g(x_1) = 0$ and $\hat g(x_2) = g(x_2)+\frac{\epsilon}{\alpha}$ has a lower objective
\begin{align*}
\Delta &=\left(\frac{\alpha g(x_1)}{x_1+\alpha g(x_1)} + \frac{\alpha g(x_2)}{x_2+\alpha g(x_2)}\right) - \left(\frac{\alpha \hat g(x_1)}{x_1+\alpha \hat g(x_1)} + \frac{\alpha \hat g(x_2)}{x_2+\alpha \hat g(x_2)}\right)\\
&=\left(\frac{\alpha g(x_1)}{x_1+\alpha g(x_1)} + \frac{\alpha g(x_2)}{x_2+\alpha g(x_2)}\right) - \frac{\alpha g(x_2) + \epsilon}{x_2+\alpha g(x_2) + \epsilon}\\
&=\frac{\epsilon}{x_1+\epsilon} - \frac{\epsilon x_2}{(x_2+\alpha g(x_2))(x_2+\alpha g(x_2)+\epsilon)} > 0
\end{align*}
Here $\Delta > 0$ and the new objective is lower since $x_2+\alpha g(x_2) > x_2$ and $x_2>x_1$ thus $(x_2+\alpha g(x_2))(x_2+\alpha g(x_2)+\epsilon) > x_2(x_1+\epsilon)$.

Thus the zero-mass region should be where $x$ is low as lower $x$ increases the objective.
Hence, the optimality will happen when $g(x)=0$ for $x\in [0,1-\tau]$, and $g(x)=1$ for $x\in[1-\tau,1 ]$. Thus: 
\begin{align*}
\int^1_{1-\tau} \frac{\alpha}{\alpha + x}\mathrm{d}x &= \alpha \int^1_{1-\tau}\frac{1}{\alpha+x}\mathrm{d}x = \alpha (\log (1+\alpha) - \log(1-\tau+\alpha))\\
&= -\alpha \log \Bigg(1-\frac{\tau}{1+\alpha}\Bigg)\\
\Longrightarrow \max_g \int^1_0 \frac{x}{x+\alpha g(x)}\mathrm{d}x &= 1 + \alpha \log \Bigg(1-\frac{\tau}{1+\alpha}\Bigg)
\end{align*}
which concludes the proof.\qed 
\end{proof}
\end{lemma}
Lemma~\ref{eq:varational_lower_bound} and Lemma~\ref{eq:varational_upper_bound} for the bounds to the AP.
\begin{theorem}
Average Precision can be bounded from above and below as following
\begin{align}
1+\alpha_c \log \Bigg(1-\frac{R_c}{1+\alpha_c}\Bigg) \ge AP_c \ge \max\Bigg(1-\sqrt{\frac{2}{3}\alpha_c R_c}, \frac{8}{9}\frac{1}{1+2\alpha_c R_c}\Bigg)
\end{align}
\begin{proof}
Let us recap the definitions of $AP$ and $R$:
\begin{align*}
AP_c &= \int^1_0 \frac{\beta}{\beta+ \alpha_c P_{x\sim D_{\lnot c}}(s_c(x)> r^{-1}_c(\beta))}\mathrm{d}\beta \\
&= \int^1_0 \frac{\beta}{\beta+ \alpha_c \underbrace{P_{x\sim D_{\lnot c}}(r_c(s_c(x)) < \beta)}_{=g(\beta)}}\mathrm{d}\beta\\
R_c 
&=\int^1_0 \underbrace{P_{x^\prime\sim D_{\lnot c}}(r_c(s_c(x^\prime)) < \beta)}_{=g(\beta)}\mathrm{d}\beta = \tau
\end{align*}
where the second line of $AP_c$ is because $r_c$ is strictly monotonously decreasing. With $x= \beta$ and $g(x) = P_{x\sim D_{\lnot c}}(r_c(s_c(x))<\beta)$, Lemma~\ref{eq:varational_lower_bound} and Lemma~\ref{eq:varational_upper_bound} are directly applicable for a function $0\le g(x)\le 1$ with a fixed $R_c=\tau$. The corresponding upper and lower bounds of $\mathcal{L}^{\text{Det}}_c$ in Theorem~\ref{thm:rank_bound} is a direct consequence of this theorem since $\mathcal{L}^{\text{Det}}_c=1-AP_c$. \qed
\end{proof}
\end{theorem}


\section{Optimal Margins}
Similar to Cao et al.~\cite{ldam}, we aim to find optimal binary margins $\gamma_+$ and $\gamma_-$ under separability condition. This reduces the problem into following:
\begin{align}
&\text{minimize}_{\gamma_+, \gamma_-}\quad \frac{1}{\gamma_+}\sqrt{\frac{1}{n_+}}+\frac{1}{\gamma_-}\sqrt{\frac{1}{n_-}}\\
&\text{subject to} \quad \gamma_++\gamma_- = 1
\end{align}
Here the constraint is due to the fact that $s_-(x)=1-s_+(x)$ in binary case and thus $\gamma_-=1-\gamma_+$. Solving the constrained optimization problem
\begin{align}
L(\gamma_+, \gamma_-, \lambda) &= \frac{1}{\gamma_+}\sqrt{\frac{1}{n_+}}+\frac{1}{\gamma_-}\sqrt{\frac{1}{n_-}} + \lambda(\gamma_++\gamma_--1)\\
\Longrightarrow \frac{\partial}{\partial \gamma_+}L(\gamma_+, \gamma_-, \lambda) &= -\frac{1}{{\gamma_+}^2}\sqrt{\frac{1}{n_+}} + \lambda=0\\
\Longrightarrow \gamma_+ &= \sqrt{\frac{{n_+}^{-\frac{1}{2}}}{\lambda}},  \quad \gamma_- = \sqrt{\frac{{n_-}^{-\frac{1}{2}}}{\lambda}}\\
\Longrightarrow \frac{\partial}{\partial \lambda}L(\gamma_+, \gamma_-, \lambda)
&= \gamma_++\gamma_--1=0\\
\Longrightarrow \gamma_++\gamma_- &= \sqrt{\frac{{n_+}^{-\frac{1}{2}}}{\lambda}}+\sqrt{\frac{{n_-}^{-\frac{1}{2}}}{\lambda}}=1\\
\Longrightarrow \sqrt{\lambda} &= \frac{\sqrt{{n_+}^{-\frac{1}{2}}}+\sqrt{{n_-}^{-\frac{1}{2}}}}{1}\\
\Longrightarrow \gamma_+ &= \frac{{n_+}^{-\frac{1}{4}}}{{n_+}^{-\frac{1}{4}}+{n_-}^{-\frac{1}{4}}}= \frac{{n_-}^{\frac{1}{4}}}{{n_+}^{\frac{1}{4}}+{n_-}^{\frac{1}{4}}}\\
 \gamma_- &= \frac{{n_-}^{-\frac{1}{4}}}{{n_+}^{-\frac{1}{4}}+{n_-}^{-\frac{1}{4}}}= \frac{{n_+}^{\frac{1}{4}}}{{n_+}^{\frac{1}{4}}+{n_-}^{\frac{1}{4}}}
\end{align}
which are as desired. The exact same process can be repeated for each class $c\in C$ and we will have our Effective Class-Margins. \qed 
\section{Surrogate Scoring Function}
In this section, we will justify the choice of our surrogate scoring function. 
\begin{align}
\hat{s}_c(x) = \frac{w^+_c e^{f(x)_c}}{w^+_c e^{f(x)_c}+w^-_c e^{-f(x)_c}}
\end{align}
The decision boundary is then 
\begin{align}
\hat{s}_c(x) = \hat{s}_{\lnot c}(x)  &= 1-\hat{s}_c(x) \\
\Longrightarrow \frac{w^+_c e^{f(x)_c}}{w^+_c e^{f(x)_c}+w^-_c e^{-f(x)_c}} &= \frac{w^-_c e^{-f(x)_c}}{w^+_c e^{f(x)_c}+w^-_c e^{-f(x)_c}}\\
\Longrightarrow \log w^+_c + f(x)_c &= \log w^-_c -f(x)_c\\
\Longrightarrow f(x)_c &= \frac{1}{2}(\log w^-_c - \log w^+_c)
\end{align}
In the unweighted sigmoid function, this point will lie at
\begin{align}
s_c(x)&=\frac{ e^{f(x)_c}}{e^{f(x)_c}+ e^{-f(x)_c}}\\
&= \frac{e^{\frac{1}{2}(\log w^-_c - \log w^+_c)}}{e^{\frac{1}{2}(\log w^-_c - \log w^+_c)} + e^{\frac{1}{2}(\log w^+_c - \log w^-_c)}}\\
&= \frac{\sqrt{\frac{w^-_c}{w^+_c}}}{\sqrt{\frac{w^-_c}{w^+_c}} + \sqrt{\frac{w^+_c}{w^-_c}}}= \frac{\sqrt{\frac{\gamma^+_c}{\gamma^-_c}}}{\sqrt{\frac{\gamma^+_c}{\gamma^-_c}} + \sqrt{\frac{\gamma^-_c}{\gamma^+_c}}}= \frac{\gamma^+_c}{\gamma^+_c + \gamma^-_c}\\
&= \gamma^+_c\\
\Longrightarrow s_{\lnot c}(x) &= \gamma^-_c
\end{align}
since $\gamma^+_c+\gamma^-_c=1$. Hence, we have shown that our surrogate scoring function with effective class-margins shifts the decision boundary of sigmoid function to $\gamma^+_c$ and $\gamma^-_c$ as desired.\qed

\begin{figure}
\begin{center}
\begin{subfigure}{.48\textwidth}
  \includegraphics[width=\linewidth,page=1]{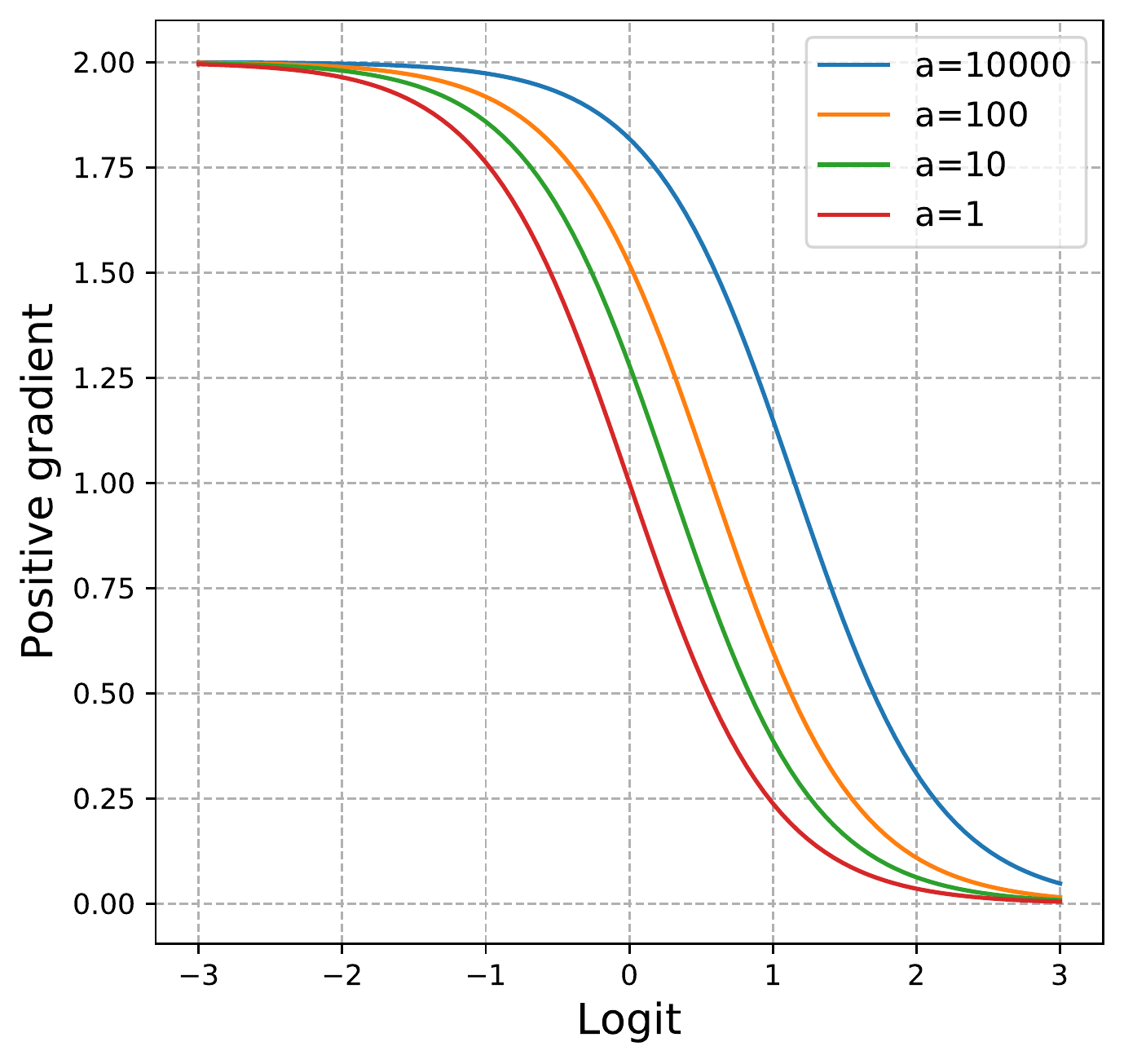}
  \label{fig:grad_pos}
\end{subfigure}
\hfill
\begin{subfigure}{.48\textwidth}
  \includegraphics[width=\linewidth,page=2]{pos_grad.pdf}
  \label{fig:grad_neg}
\end{subfigure}
\end{center}
\caption{Visualization of the positive and negative gradients from ECM Loss with different positive and negative samples ratios. }
\label{fig:grad_plots}
\end{figure}

\begin{figure}[t]
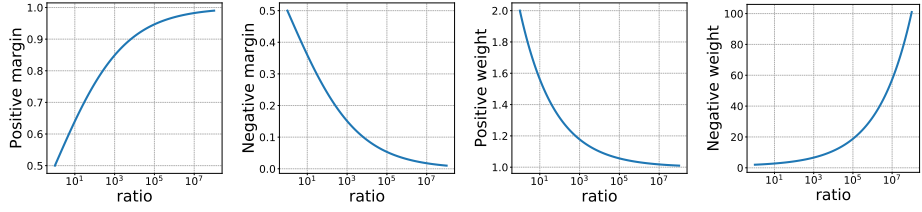

\begin{center}
\begin{subfigure}{.24\textwidth}
  \includegraphics[width=\linewidth,page=1]{pos_m.pdf}
\end{subfigure}
\hfill
\begin{subfigure}{.24\textwidth}
  \includegraphics[width=\linewidth,page=2]{pos_m.pdf}
\end{subfigure}
\hfill
\begin{subfigure}{.24\textwidth}
  \includegraphics[width=\linewidth,page=3]{pos_m.pdf}
\end{subfigure}
\hfill
\begin{subfigure}{.24\textwidth}
  \includegraphics[width=\linewidth,page=4]{pos_m.pdf}
\end{subfigure}
\end{center}
\caption{Visualization of the positive and negative margins and weights as a function of the sample ratio $\alpha_c$. }
\label{fig:mw_plots_supp}
\end{figure}

\begin{figure}[t]
\begin{center}
\begin{subfigure}{.48\textwidth}
  \includegraphics[width=\linewidth,page=1, trim={1.5cm 0 1.5cm 0}]{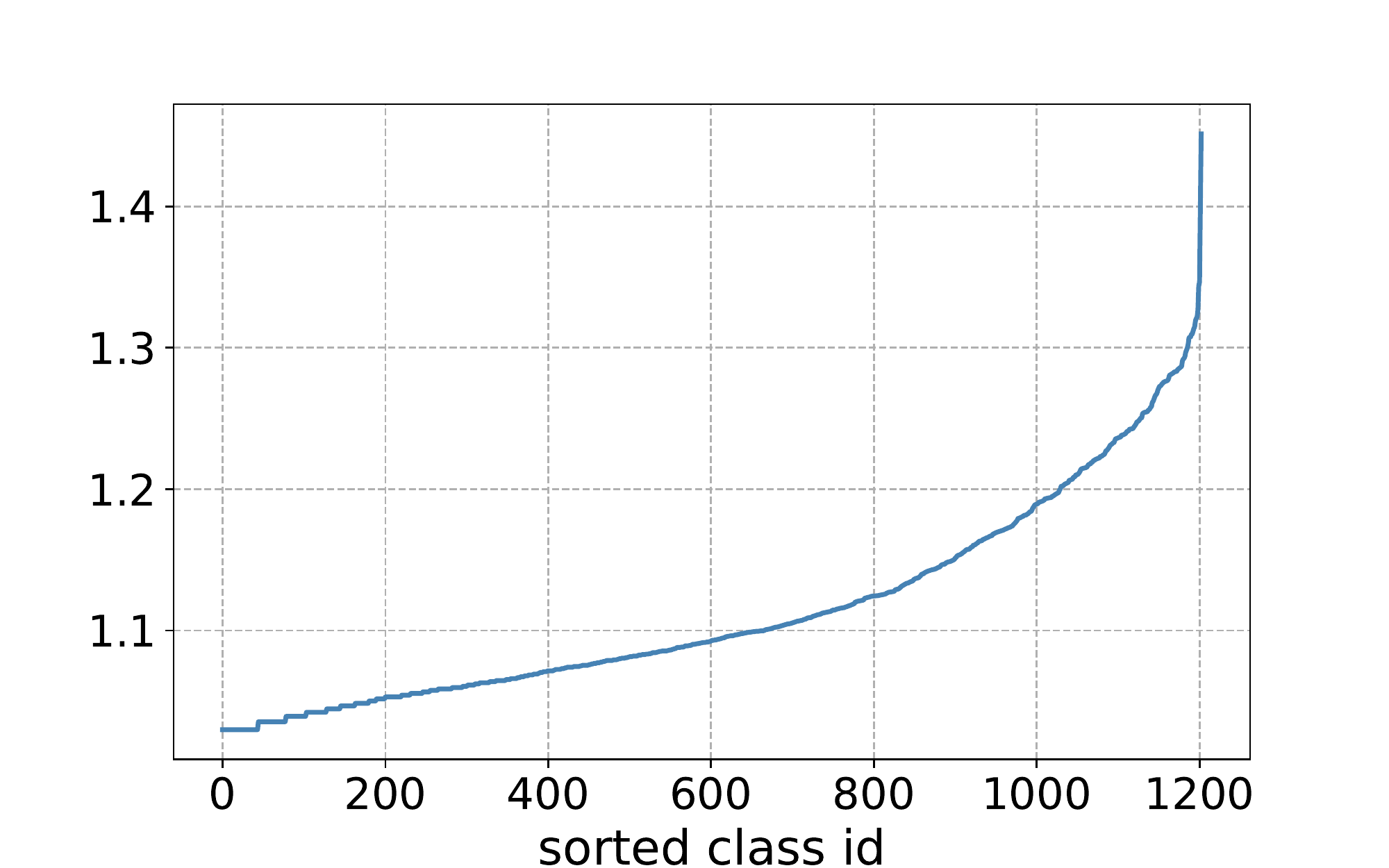}
  \caption{$w_c^+$}
  \label{fig:w_positive}
\end{subfigure}
\hfill
\begin{subfigure}{.48\textwidth}
  \includegraphics[width=\linewidth,page=2, trim={1.5cm 0 1.5cm 0}]{w_positive.pdf} \caption{$w_c^-$}
  \label{fig:w_negative}
\end{subfigure}
\end{center}
\caption{Visualization of the computed the positive and negative weights used in our scoring function $\hat{s}_c$ for each class $c\in C$. The prior distribution is from LVIS v1 training annotations over 1203 classes. The weights are sorted in ascending order and background probability is measured with Mask R-CNN with ResNet-50 backbone.}
\label{fig:margin_plot}
\end{figure}

\section{Margins vs Weights vs Gradients}
In this section, we provide more intuitions about the relationship between margins, weights, and gradients. The gradient of positive and negative samples with ECM Loss is the following:
\begin{align}
\frac{\partial}{\partial f(x)_c} \ell_{\text{ECM}}(x, y) 
&= \frac{2w_{\lnot c} e^{f(x)_{\lnot c}}}{w_c e^{f(x)_c} + w_{\lnot c} e^{f(x)_{\lnot c}}} \propto w_{\lnot c} = \frac{1}{\gamma_{\lnot c}} \propto n_{\lnot c}\\
\frac{\partial}{\partial f(x)_{\lnot c}} \ell_{\text{ECM}}(x, y) 
&= \frac{2w_c e^{f(x)_c}}{w_c e^{f(x)_c} + w_{\lnot c} e^{f(x)_{\lnot c}}} \propto w_c\:\; = \frac{1}{\gamma_c}\:\:\: \propto n_c
\end{align}
where we omit detection weight $m_c$ for simplicity. The positive gradient is greater for rare classes (higher $n_{\lnot c}$) compared to frequent classes, whereas the negative gradient is lower. This coincides with the intuitions from prior works based on heuristics or indirect measure of a model. Below, we show visualization of positive and negative gradients as a function of logit with different positive and negative ratios $a_c=\frac{n_{\lnot c}}{n_c}$. In Figure~\ref{fig:grad_plots}, low $a$ means frequent classes whereas high $a$ means rare classes. Our surrogate scoring function $\hat{s}_c$ balances the gradient values based on the frequency of each class. Frequent classes get lower positive gradient and higher negative gradient (red in~\ref{fig:grad_plots}) whereas rare classes get higher positive gradient and lower negative gradient (blue in~\ref{fig:grad_plots}). In Figure~\ref{fig:mw_plots_supp}, we further visualize the relationship between the positive and negative margins and weights as a function of the ratio $\alpha_c$.
In Figure~\ref{fig:margin_plot}, we visualize the computed weights for our scoring function $\hat{s}_c$ in Equation~\ref{eq:surrogate_s} for $w^+_c$ (left) and $w^-_c$ (right).


\section{Implementation Details}
In this section, we will discuss details of the experiments and implementation. 

\myparagraph{Background count.}
We empirically measure the frequency of background samples as a ratio of foreground and background samples within a batch, $r$, in the \textit{classification layer} of a detector during the first few iterations. This ratio will then be used to derive \textit{dataset-level} count of background samples as
\begin{align}
    n_{bg}=r\sum_{c\in C}n_c
\end{align}
where $n_c$ is the number of positive samples of class $c$ in the training dataset. Then, we compute the sample count of each class \textit{with background} as
\begin{align}
    n^+_c &= n_c\\
    n^-_c &= \Big(\sum_{c^\prime \in C\cup \{bg\}} n_{c^\prime}\Big) - n_c
\end{align}
and use it to compute the effective class-margins. For one-stage detectors, we only count for foreground classes as foreground and background imbalance is managed from focal weight~\cite{focal}. 

\myparagraph{Two-stage detectors.}
We train two-stage instance segmentation models based on Mask R-CNN~\cite{maskrcnn} and Cascade Mask R-CNN~\cite{cascadercnn} with various backbones, ResNet-50 and ResNet-101~\cite{resnet} with Feature Pyramid Network~\cite{fpn} pretrained on ImageNet-1K~\cite{imagenet}, Swin Transformer~\cite{swin} pretrained on ImageNet-21K with 224x224 image resolution.
We train with for 12 or 24 epochs with Repeated Factor Sampler (RFS)~\cite{lvis} on a $1\times$ or $2\times$ schedule respectively. 
For CNN backbones, we use the SGD optimizer with 0.9 momentum, initial learning rate of 0.02, weight decay of 0.0001, with step-wise scheduler decaying learning rate by 0.1 after 8 and 11 epochs for 1x, and 20 and 22 epochs for 2x, and batch size of 16 on 8 GPUs. Please note that the baseline methods are trained with their optimal learning rate schedule with decaying schedule of 16 and 22 epochs. For example, Mask R-CNN with ResNet-50 trained with Seesaw Loss~\cite{seesaw} on decaying schedule of 20 and 22 epochs result with 26.7 $\mathrm{mAP}_{\mathrm{segm}}$ and 26.9 $\mathrm{mAP}_{\mathrm{bbox}}$, whereas decaying schedule of 16 and 22 (default) result with 26.7 $\mathrm{mAP}_{\mathrm{segm}}$ and 27.3 $\mathrm{mAP}_{\mathrm{bbox}}$. 
For Transformer backbones, we use the AdamW optimizer with an initial learning rate of $0.00005$, beta set to $(0.9, 0.999)$, weight decay of $0.05$, with Cosine-annealing scheduler.
For all our models, we follow the standard data augmentation during training: random horizontal flipping and multi-scale image resizing to fit the shorter side of image to ($640, 672, 704, 736, 768, 800$) at random, and the longer side kept smaller than $1333$.
For Swin Transformer, we use a larger range of scale of the short side of image from $480$ to $800$.
For two-stage detectors, we normalize the classification layers with temperature $\tau=20$ for both box and mask classifications, and apply foreground calibration as post-process following prior practices~\cite{seesaw,disalign,ap_fixed}.
LVIS has more instances than COCO.
We thus increase the per-image detection limit to 300 from 100 and set the confidence threshold to 0.0001.
This is common practice in LVIS~\cite{lvis}.
For OpenImages, we train Cascade R-CNN with ResNet-50 backbone following the baseline and experimental setup of Zhou~\etal~\cite{unidet}.
All models in this experiment were trained for 180k iterations with a class-aware Sampler.

\myparagraph{One-stage detectors.}
For one-stage detectors, Focal Loss~\cite{focal} is the standard choice of the loss function.
It effectively diminishes loss values for ``easy'' samples such as the background.
In this experiment, we test the compatibility of our method with Focal weights.
Specifically, we apply the computed focal weights to our ECM Loss.
Instead of a binary cross-entropy on the surrogate scoring function $\hat{s}$, we use the focal loss.
We use a number of popular one-stage detectors: FCOS~\cite{fcos}, ATSS~\cite{atss} and VarifocalNet~\cite{varifocal}.
Each method uses either Focal Loss or a variant~\cite{varifocal}.
We use the default hyperparameter for all types of focal weights: $\gamma=2, \alpha=0.25$ for Focal Loss, $\gamma=2$ and $\alpha=0.75$ for Varifocal Loss~\cite{varifocal}.
In LVIS v1.0 models expect to see more instances.
We thus double the per-pyramid level number of anchor candidates from $9$ to $18$ for ATSS and VarifocalNet.
Similar to 2-stage detectors, we increased per-image detection to 300 and set the confidence threshold to 0.0001.
We train on ResNet-50 and ResNet-101 backbones for 12 epochs for 1x and 24 epochs for 2x, batch size of 16 on 8 GPUs. We set the learning rate to be 0.01 which was the optimal learning rate for the baselines. For all other settings, we follow the two-stage experiments. 

\clearpage

\end{document}